\documentclass[conference,12pt,onecolumn]{ieeeconf} 

\IEEEoverridecommandlockouts                              

\title{\LARGE \bf
Actor-Critic or Critic-Actor? A Tale of Two Time Scales
}

\author{Shalabh Bhatnagar$^{1}$, Vivek S.\ Borkar$^{2}$, and Soumyajit Guin$^{1}$
\thanks{SB was supported by a J.~C.~Bose Fellowship, Project No.~DFTM/ 02/ 3125/M/04/AIR-04 from DRDO under DIA-RCOE, a project from DST-ICPS, and the RBCCPS, IISc.}
\thanks{VSB was supported by the S.~S.~Bhatnagar Fellowship from the Council of Scientific and Industrial Research, Government of India.}
\thanks{$^{1}$Department of Computer Science and Automation, Indian Institute of Science, Bengaluru 560012, India
        {\tt\small shalabh@iisc.ac.in, gsoumyajit@iisc.ac.in}}%
\thanks{$^{2}$Department of Electrical Engineering, Indian Institute of Technology Bombay, Mumbai 400076, India
        {\tt\small borkar.vs@gmail.com}}%
}
\usepackage{enumerate}
\usepackage{amssymb}
\usepackage{amsmath}
\usepackage{mathrsfs}
\usepackage{graphicx}
\usepackage{subfigure}
\usepackage{color}

\usepackage{multirow}
\usepackage{mathtools}

\newtheorem{thm}{Theorem}

\newtheorem{lem}[thm]{Lemma}

\newtheorem{prop}[thm]{Proposition}
\newtheorem{remark}[thm]{Remark}

\newtheorem{assum}{Assumption}

\usepackage{multirow}
\usepackage{mathtools}

\newcommand{\F}{\mathcal{F}}

\newcommand{\M}{\mathcal{M}}

\usepackage[ruled,vlined]{algorithm2e}

\begin{document}

\baselineskip=24pt

\maketitle

\begin{abstract}
 We revisit the standard formulation of tabular actor-critic algorithm as a two time-scale stochastic approximation with value function computed on a faster time-scale and policy computed on a slower time-scale. This emulates  policy iteration. We observe that reversal of the time scales will in fact emulate value iteration and is a legitimate algorithm. We provide a proof of convergence and compare the two empirically with and without function approximation (with both linear and nonlinear function approximators) and observe that our proposed critic-actor algorithm performs on par with actor-critic in terms of both accuracy and computational effort.

\medskip


\end{abstract}

\section{Introduction}
\label{introduction}
The actor-critic algorithm of Barto et al.\ \cite{BartoSA} is one of the foremost reinforcement learning algorithms for data-driven approximate dynamic programming for Markov decision processes. Its rigorous analysis as a two time-scale stochastic approximation was initiated in \cite{KondaB} and \cite{KondaT}, first in tabular form, then with linear function approximation, respectively. The algorithm has a faster time scale component for value function evaluation (the `critic'),  with policy evaluation on a slower time scale (the `actor'). Using two time-scale philosophy, the former sees the latter as quasi-static, i.e., varying slowly enough on the time-scale of the critic that critic can treat the actor's output essentially as a constant. In turn, the actor sees the critic as quasi-equilibrated, i.e., tracking the value function corresponding to the current policy estimate of the actor. Thus the scheme emulates policy iteration, wherein one alternates between value function computation for a fixed policy and policy update based on this value function by minimization of the associated `Q-value'. Another interpretation is that of a Stackelberg game between the actor and the critic, with the actor leading and being followed by the critic. That is, the actor tunes the policy slowly and the critic reacts to it rapidly, so that the actor can factor in the critic's reaction in her update.

This raises the issue as to what would happen if the time scales of the actor and the critic are reversed, rendering critic the leader. We observe below that in this case, the scheme emulates value iteration, making it a valid reinforcement learning scheme.  We call it the `critic-actor' algorithm. This structure as such is not novel and was  explored in the context of Q-learning in \cite{ShalabhB}, \cite{ShalabhL} with a different motivation. Its incorporation into the actor-critic facilitates a dimensionality reduction and a clean convergence analysis. {This is so because we work with value and not Q-value functions.}

While there are multiple variations of the actor-critic based on the exact scheme used for policy updates, we stick to one of the three proposed in \cite{KondaB} to make our point.  It may be recalled that one of the purported advantages of the actor-critic algorithm is that because of the slow time scale of the policy update which renders it quasi-static, the value function update is `essentially' a linear operation and therefore one can legitimately use schemes such as TD$(\lambda)$  \cite{Sutton}, \cite{TsitsiklisV} that are based on linear function aproximation. This is in contrast with, e.g., Q-learning \cite{Watkins} where the nonlinearity of the iterate is not linear function approximation friendly. This problem returns to actor-critic if we interchange the time scales. Nevertheless, given the current trend towards using neural networks for function approximation, this theoretical advantage is no longer there. In particular, this puts the actor-critic and critic-actor schemes a priori on equal footing as far as nonlinear function approximation is concerned.

\section{The Basic Framework}
\label{framework}

Our Markov decision process (MDP) is a random process $X_n,n\geq 0$, in a finite state  space $S$ satisfying, a.s.,
\[P(X_{n+1}=j\mid X_k, A_k, k\leq n) = p(X_n, A_n,j), \ \forall n\geq 0.
\]
Here {$A_n\in U(X_n)$ is the action at time $n$ when the state is $X_n$} where $U(i) :=$ the finite set of actions  admissible 
in state $i$. 
Also, $p(i,a,j)$ denotes the transition probability from state $i$ to $j$ when a feasible action $a$ is chosen in state $i$. {$\varphi := \{A_n\}$ with each $A_n$ as above will be said to be admissible.}
For simplicity, we take $U(\cdot) \equiv U$. 
A stationary deterministic policy (SDP) $f$ is one where $A_n = f(X_n)$ for some $f:S\rightarrow U$. Similarly, if given $X_n$, $A_n$ is conditionally independent of $\{X_m,A_m, m<n\}$
and  has the same conditional law $\pi: S\rightarrow \mathcal{P}(U)$ $\forall n$, then $\pi$ is called a stationary randomised policy (SRP).
Here $\mathcal{P}(U) :=$ the space of probability vectors on $U$. We shall denote
the probability vector $\pi(i)$ as $\pi(i) = (\pi(i,a),a\in U(i))^T$. Let $g: S\times U\times S\rightarrow \mathbb{R}$ denote the single-stage cost function
and $\gamma\in (0,1)$  the discount factor.

{For a given admissible sequence $\varphi = \{A_n\}$, consider the infinite horizon discounted cost
\begin{equation}
	\label{vf}
	V_{\varphi}(i) \stackrel{\triangle}{=} E\left[ \sum_{n=0}^{\infty} \gamma^n g(X_n, A_n, X_{n+1})\mid X_0=i\right], \mbox{ } i\in S.
\end{equation}
The function $V_{\pi} := V_{\varphi}$ for $\varphi\approx$ an SRP $\pi$  is called the value function under the SRP $\pi$.}
 The corresponding (linear) Bellman equation takes the form
\begin{equation}
	\label{bepi}
	V_\pi(i) = \sum_{a\in U(i)} \pi(i,a)\sum_{j\in S} p(i,a,j)(g(i,a,j) + \gamma V_\pi(j)).
\end{equation}
Define the `value function'
\begin{equation}
	\label{vfo}
{V^*(i) = \min_{\varphi} V_{\varphi}(i).}
\end{equation}
This satisfies the Bellman equation
\begin{equation}
	\label{be}
	V^*(i) = \min_{a\in U(i)} \left( \sum_{j\in S}p(i,a,j)(g(i,a,j) +\gamma V^*(j)\right).
\end{equation}
Define the Q-value function $Q^*(\cdot,\cdot): S\times U \to\mathcal{R}$ so that $Q^*(i,a) :=$ the total cost incurred if starting in state $i$, action $a$ is chosen and  the optimal action is chosen in each state
visited subsequently, i.e.,
\begin{equation}
	\label{Qv}
	Q^*(i,a) = \sum_{j\in S}p(i,a,j)(g(i,a,j) +\gamma V^*(j)).
\end{equation}
Then an SRP $\pi^*$ would be optimal if support$(\pi^*(i))$ $ \subset \arg\min Q^*(i,\cdot)$.
In particular, an optimal SDP (hence SRP) exists and is given by any choice of minimiser of $Q^*(i, \cdot )$ for each $i$. See \cite{Puterman} for an extensive treatment.

Policy iteration (PI) and value iteration (VI) are two of the numerical procedures for solving the Bellman equation. 
Whereas actor-critic algorithms  \cite{KondaB} emulate PI, the algorithm  with the timescales of the actor-critic algorithm reversed will be seen to mimic VI.

\section{The Proposed Critic-Actor Algorithm}
\label{acca}

Let $V_n(\cdot)$ and $\pi_n(\cdot,\cdot)$ denote the estimates at instant $n$ of the value function and the optimal SRP. 
%
We consider here an off-policy setting where states and state-action tuples are sampled from a priori given distributions in order 
to decide the particular state whose value is updated next, as well as the action-probability estimate (in the SRP) of the sampled state-action tuple to be updated. 
This is more general than the on-policy setting commonly considered in standard actor-critic algorithms such as \cite{KondaT,ShalabhSGL},
as the latter  turns out to be a special case of the off-policy setting we consider. 

Let $\{Y_n\}$ and $\{Z_n\}$ be $S$ and $S\times U$ valued processes such that if $Y_n=i\in S$, then the value of state $i$, denoted by $V_n(i)$, is updated
at the $n$th iterate. Likewise if $Z_n =(i,a)$, then the policy component corresponding to the $(i,a)$-tuple is updated. Define $\{\nu_1(i,n)\}, \{\nu_2(i,a,n)\}$  by (for $n>0$):
\[\nu_1(i,n) = \sum_{m=0}^{n-1} I\{Y_m =i\}, \mbox{ with } \nu_1(i,0)=0,
\]
\[
\nu_2(i,a,n) = \sum_{m=0}^{n-1} I\{Z_m=(i,a)\}, \mbox{ with } \nu_2(i,a,0)=0.
\]

As with policy gradient schemes \cite{SuttonMSM}, we parameterize the policy.  Specifically we consider
a parameterised Boltzmann or Gibbs form for the SRP as below.
\begin{equation}
	\label{gibbs}
	\pi_\theta(i,a) = \frac{\exp(\theta(i,a))}{\sum_{b}\exp(\theta(i,b))}, \mbox{ } i\in S, a\in U.
\end{equation}
{For a given $\theta_0\gg 0$, let $\Gamma_{\theta_0}: \mathbb{R} \to [-\theta_0,\theta_0]$ denote the projection map.} 
Let $\{a(n)\}, \{b(n)\}$ be positive step size sequences satisfying conditions we specify later.



\subsection*{The Critic-Actor Algorithm}
\label{caa}
The proposed critic-actor algorithm has similar set of updates as Algorithm 3 of \cite{KondaB} but with the timescales reversed, i.e., $a(n) = o(b(n))$.
Let $\{\xi_n(i,a)\}$ and $\{\eta_n(i,a)\}$, $i \in S$, $a\in U$ be independent families of i.i.d random variables with
law $p(i,a,\cdot)$. Let $\{\phi_n(i)\}$ be a sequence of $U$-valued i.i.d random variables with the conditional
law of $\phi_n(i)$ given {the sigma-algebra $\sigma(V_m(\cdot), \pi_m(\cdot), \xi_m(\cdot,\cdot)$, $\eta_m(\cdot,\cdot), m\leq n)$ generated by the random variables realised till $n$, 
being denoted by $\pi_n(i,\cdot)$.}
The critic and actor
recursions now take the following form:
\begin{equation}
\label{critic2}
\begin{split}
V_{n+1}(i) &= V_n(i) + a(\nu_1(i,n))[g(i,\phi_n(i),\xi_n(i,\phi_n(i)))\\
&+ \gamma V_n(\xi_n(i,\phi_n(i)))- V_n(i)] I\{Y_n=i\},
\end{split}
\end{equation}
\begin{equation}
\label{actor2}
\begin{split}
&\theta_{n+1}(i,a) = \Gamma_{\theta_0}(\theta_n(i,a) + b(\nu_2(i,a,n))[V_n(i) \\
&- g(i,a,\eta_n(i,a))-\gamma V_n(\eta_n(i,a))]I\{Z_n=(i,a)\}).
\end{split}
\end{equation}

The $V_n$ update is governed by the step-size sequence $a(n),n\geq 0$, while the
$\theta_n$ update is governed by $b(n),n\geq 0$. From Assumption~\ref{a1} below, this implies that the $V_n$ update proceeds on a slower timescale
in comparison to the $\theta_n$ update.

When compared to Q-learning, the actor here replaces the exact minimisation in each Q-learning iterate, not always easy, by a recursive scheme for the same and can be seen as a subroutine for minimization, albeit executed concurrently using a faster time scale to get the same effect.
In the next section, we give a proof of convergence of our critic-actor algorithm.

\section{Convergence of Critic-Actor  Scheme}
\label{convergence}

We consider the ordinary differential equation (ODE) approach for proving convergence of the stochastic approximation recursion originally due to Derevitskii-Fradkov and Ljung \cite{Ljung} and subsequently also studied in many other references, cf.~ \cite{Benaim, Benveniste, BorkarB, KushnerC, KushnerY}.
The idea here is to show that a suitable piecewise  linear interpolation of the iterates asymptotically tracks an associated o.d.e.~in an appropriate sense as time tends to infinity. As can be seen from the several textbook treatments of this approach as referenced above, this approach has become fairly standard in the stochastic approximation community.

For $x>0$, define 
\begin{eqnarray*}
N_1(n,x) &=& \min\{ m>n \mid \sum_{i=n+1}^{m} \bar{a}(i) \geq x\},\\
N_2(n,x) &=& \min\{ m>n \mid \sum_{i=n+1}^{m} \bar{b}(i) \geq x\},
\end{eqnarray*}
{where $\bar{a}(n) \equiv a(\nu_1(i,n))$ and $\bar{b}(n) \equiv b(\nu_2(i,a,n))$, respectively.}
We make the following assumptions.
\begin{assum}[Step-Sizes]
	\label{a1}
	$a(n),b(n),n\geq 0$ {are eventually non-increasing} and satisfy the following conditions:
	\begin{itemize}
		\item[(i)] ${\displaystyle \sum_n a(n)=\sum_n b(n) =\infty}$.
		\item[(ii)] ${\displaystyle \sum_n (a(n)^2 + b(n)^2)<\infty}$. 
		\item[(iii)] $a(n)=o(b(n))$.
		\item[(iv)] For any $x\in (0,1)$, ${\displaystyle \sup_n \frac{a([xn])}{a(n)}}$, ${\displaystyle \sup_n \frac{b([xn])}{b(n)}<\infty},$ where $[xn]$ denotes the integer part of $xn$.
	   \item[(v)] For any $x\in (0,1)$, ${\displaystyle A(n) := \sum_{i=0}^{n} a(i)}$, ${\displaystyle B(n) := \sum_{i=0}^{n} b(i)}$, $n\geq 0$, we have
	   ${\displaystyle \frac{A([yn])}{A(n)}}$, ${\displaystyle \frac{B([yn])}{B(n)} \rightarrow 1}$, as $n\rightarrow\infty$, 
	   uniformly over $y\in [x,1]$.
		\end{itemize}
\end{assum}

\begin{assum}[Frequent Updates]
	\label{a2}	
	\begin{itemize}
		\item[(i)]  
	  There exists a constant $\kappa >0$ such that the following conditions hold almost surely: 
${\displaystyle 
\liminf_{n\rightarrow\infty} \frac{\nu_1(i,n)}{n} \geq \kappa}$, $\forall i\in S$, 
${\displaystyle \liminf_{n\rightarrow\infty} \frac{\nu_2(i,a,n)}{n} \geq \kappa}$, $\forall (i,a)\in S\times U$.

\item[(ii)] For $x>0$, the limits
\[\lim_{n\rightarrow\infty} \frac{\sum_{j=\nu_1(i,n)}^{\nu_1(i,N_1(n,x))} a(j)}{\sum_{j=\nu_1(k,n)}^{\nu_1(k,N_1(n,x))} a(j)}, \mbox{   }
\lim_{n\rightarrow\infty} \frac{\sum_{j=\nu_2(i,a,n)}^{\nu_2(i,a,N_2(n,x))} b(j)}{\sum_{j=\nu_2(k,b,n)}^{\nu_2(k,b,N_2(n,x))} b(j)},
\]
exist almost surely for states $i,k$ and state-action tuples $(i,a)$ and $(k,b)$ in the two limits respectively.	
	\end{itemize}
	
	\end{assum}
{
Since $\nu_1(i,n)\leq n$, $\forall i$, and 
$\nu_2(i,a,n)\leq n$, $\forall (i,a)$ tuples, we have, for $n$ sufficiently large, $a(n) \leq \bar{a}(n)$, and $b(n) \leq \bar{b}(n)$, respectively}. Assumption 1(i)-(ii) are the standard Robbins-Monro conditions for stochastic approximation algorithms (see Ch.~2 of \cite{BorkarB}). Condition (iii) is standard for two time-scale stochastic approximations (Section 8.1, \cite{BorkarB}). Conditions (iv), (v) and Assumption 2 are required for handling the asynchronous nature of the iterates (see Ch.~6 of \cite{BorkarB}, also \cite{BorkarA}). {Examples of step-sizes that satisfy Assumptions 1 and 2(ii) include (see \cite{KondaB}) (i) $a(n)=1/(n+1), b(n)= \log(n+2)/(n+2)$, (ii) $a(n)=1/((n+2)\log(n+2))$, $b(n)=1/(n+1)$, etc., for $n\geq 0$. 
Assumption 2(i) is satisfied for instance if $\{Z_n\}$ is ergodic and $\{Y_n\}$ is ergodic under any given policy dictated by $\{\theta_n\}$. If these are independently sampled from some distributions, the same should assign positive probabilities to all state and state-action components. }


\subsection{Convergence of the Faster Recursion}

It has been shown in Lemma 4.6 of \cite{KondaB} that $\bar{a}(n), \bar{b}(n) \rightarrow 0$ as $n\rightarrow \infty$. Further, $\bar{a}(n) =o(\bar{b}(n))$. 
Since our $V$-update proceeds on the slower timescale, we let $V_n \approx V$ for the analysis of the faster recursion.
For all $(i,a)$, let 
\begin{eqnarray*}
g_{ia}(V) &=& V(i) -\sum_{j} p(i,a,j)(g(i,a,j) + \gamma V(j)),\\
k_{ia}(V) &=& \sum_{j} p(i,a,j)(g(i,a,j) + \gamma V(j)) - V(i).
\end{eqnarray*}
Let ${\mathcal{F}}_n \stackrel{\triangle}{=} \sigma(\xi_m(i,a),\eta_m(i,a),\phi_m(i), m<n;$ $Y_m, Z_m, V_m,$ $\theta_m(i,a),$ $m\leq n, i\in S, a\in U), n \geq 0,$ denote an increasing family of associated sigma fields. Let
\[
\begin{split}
&\mu^2_n(i,a) = I\{Z_n=(i,a)\}, \mbox{ } \forall (i,a) \in S\times U,\\
&\mu^1_n(i) = I\{Y_n=i\}, \mbox{ } \forall i \in S,\\
&\M^{\mu^2_n} = \mbox{diag}(\mu^2_n(i,a), (i,a)\in S\times U),\\
&\M^{\mu^1_n} = \mbox{diag}(\mu^1_n(i), i\in S).
\end{split}
\]
{For $n\geq 0$, define the $\{\F_n\}$-adapted sequences
\begin{eqnarray*}
M_{n}(i,a) &=& V(i) - g(i,a,\eta_{n-1}(i,a)) -  \gamma V(\eta_{n-1}(i,a))\\
&& + \ k_{ia}(V), \\
N_{n}(i) &=& g(i,\phi_{n-1}(i),\xi_{n-1}(i,\phi_{n-1}(i) \linebreak )) -V(i)\\
&& + \gamma V(\xi_{n-1}(i,\phi_{n-1}(i)))
- \sum_{a\in U} \pi_{\theta_n}(i,a) k_{ia}(V).
\end{eqnarray*}}
		
\begin{lem}
	\label{lem1}
The sequences 
$${\displaystyle	\sum_{n=1}^{m} b(\nu_2(i,a,n))M_{n+1}(i,a)} \linebreak \displaystyle {I\{Z_n=(i,a)\}}, \ m\geq 1, \ \mbox{and}$$
 $${\displaystyle	\sum_{n=1}^{m} a(\nu_1(i,n))N_{n+1}(i) I\{Y_n} \linebreak \displaystyle{=i\}}, \ m\geq 1,$$ 
converge almost surely as $m\rightarrow\infty$.
\end{lem}

\begin{proof} 
	Follows as in Lemma 4.5 of \cite{KondaB}.
	\end{proof}

The analysis of \eqref{critic2} proceeds by rewriting it as:
\begin{eqnarray*}
\theta_{n+1}(i,a) &=& \Gamma_{\theta_0}(\theta_n(i,a) + \bar{b}(n) \M^{\mu^2_n}
[V(i) \\
&&- \ g(i,a,\eta_n(i,a))-\gamma V(\eta_n(i,a))]).
\end{eqnarray*}

One may further rewrite it as follows:
\begin{equation*}
\label{actor1-2}
\begin{split}
\theta_{n+1}(i,a) =& \Gamma_{\theta_0}(\theta_n(i,a) + \bar{b}(n) \M^{\mu^2_n} [g_{ia}(V)+ M_{n+1}(i,a)]).
\end{split}
\end{equation*}
Let
\[\bar{\Gamma}_{\theta_0}(g_{ia}(V)) = \lim_{\Delta\rightarrow 0} \left(\frac{\Gamma_{\theta_0}(\theta(i,a) + \Delta g_{ia}(V)) - \theta(i,a)}{\Delta}\right).
\]
Consider now the ODE
\begin{equation}
	\label{ode-a}
	\dot{\theta}(i,a) = \bar{\Gamma}_{\theta_0}(g_{ia}(V)).
\end{equation}
Let 
\[
\gamma_{ia}(\theta) =\left\{
\begin{array}{ll}
0 & \mbox{ if } \theta(i,a) = \theta_0 \mbox{ and } k_{ia}(V) \leq 0,\\
   & \mbox{ or } \theta(i,a) = -\theta_0 \mbox{ and } k_{ia}(V) \geq 0,\\
1  & \mbox { otherwise.}
\end{array}
\right.
\]
Then it can be seen (see Sec.5.4 of \cite{KondaB}) that 
\[\bar{\Gamma}_{\theta_0}(g_{ia}(V)) = -k_{ia}(V)\gamma_{ia}(\theta).
\]
Thus, the ODE \eqref{ode-a} takes the form:
\begin{equation}
	\label{ode-b}
	\dot{\theta}(i,a) = -k_{ia}(V)\gamma_{ia}(\theta).
\end{equation}

From the parameterised form of the policy \eqref{gibbs}, it follows from \eqref{ode-b} that for $(i,a)\in S\times U$, 
\begin{equation*}
	\dot{\pi}_\theta(i,a) = \pi_\theta(i,a)(\dot{\theta}(i,a) -  \sum_{b\in U} \pi_\theta(i,b) \dot{\theta}(i,b)),
\end{equation*}
\[
=  \pi_\theta(i,a) (\bar{\Gamma}_{\theta_0}(g_{ia}(V)) - \sum_{b\in U} \pi_\theta(i,b)
\bar{\Gamma}_{\theta_0}(g_{ib}(V))),
\]
\[
=  -\pi_\theta(i,a) (k_{ia}(V)\gamma_{ia}(\theta) - \sum_{b\in U} \pi_\theta(i,b)
k_{ib}(V)\gamma_{ib}(\theta)).
\]

This is a replicator dynamics whose stable attractors are sets consisting of the $\pi_\theta$ that minimize
${\displaystyle \sum_{(i,a)\in S\times U} k_{ia}(V)}$, i.e., $\theta$ that satisfy $\theta(i,a) =\theta_0$ for at least
one $a\in \arg\min(k_{i\cdot}(V))$ and $\theta(i,a)=-\theta_0$ otherwise (see Lemma 5.10 of \cite{KondaB})\footnote{Under reasonable conditions on noise, it is known that stochastic approximation converges a.s.\ to its stable attractors, see, e.g., section 3.4 of \cite{BorkarB} and the references therein.}. Let $\hat{\pi}$ denote one such policy and
$\hat{V}$ the corresponding value function. 
Then
\[\hat{V}(i) = \sum_{a\in U} \hat{\pi}(i,a) \sum_{j}p(i,a,j)(g(i,a,j) + \gamma \hat{V}(j)), \mbox{ } \forall i.
\]


Define a policy $\pi^*$ that corresponds to $\theta_0\rightarrow\infty$
in the above. Then this is an optimal policy and the corresponding value function $V^*$ satisfies the corresponding
dynamic programming equation. 
\[{V^*}(i) = \sum_{a\in U} {\pi^*}(i,a) \sum_{j}p(i,a,j)(g(i,a,j) + \gamma {V^*}(j)), \mbox{ } \forall i.
\]

\begin{prop}
	\label{prop1}
	We have 
	\[
	\|\hat{V}-V^*\|_\infty \leq \frac{\max_i\|\hat{\pi}(i,\cdot)-\pi^*(i,\cdot)\|_1\max_i\|\bar{g}(i,\cdot)\|_\infty}{(1-\gamma)^2},
	\]
	where ${\displaystyle \bar{g}(i,a) = \sum_{j}p(i,a,j)g(i,a,j)}$.
	\end{prop}
\begin{proof}
Note that
\begin{eqnarray}
\lefteqn{|\hat{V}(i) - V^*(i)| \leq \sum_a |\hat{\pi}(i,a)-\pi^*(i,a)| |\bar{g}(i,a)| +}\nonumber \\
&&\gamma \sum_a \hat{\pi}(i,a)\sum_j p(i,a,j) |\hat{V}(j)-V^*(j)| + \nonumber\\
&&\gamma \sum_a |\hat{\pi}(i,a)-\pi^*(i,a)|\sum_{j}p(i,a,j)|V^*(j)|. \label{eq-1}
\end{eqnarray}
Thus,
\[
\|\hat{V}-V^*\|_\infty \leq \max_i \|\hat{\pi}(i,\cdot)-\pi^*(i,\cdot)\|_1\max_i\|\bar{g}(i,\cdot)\|_\infty
\]
\begin{equation}
	\label{eq-2}
	+ \ \gamma \|\hat{V}-V^*\|_\infty + \gamma \max_i \|\hat{\pi}(i,\cdot) - \pi^*(i,\cdot)\|_1\|V^*\|_\infty.
\end{equation}
Now note that by definition
\begin{equation}
	\label{eq-3}
	\|V^*\|_\infty \leq (1-\gamma)^{-1} \max_i\|\bar{g}(i,\cdot)\|_\infty.
\end{equation}
The claim follows upon substituting \eqref{eq-3}  in \eqref{eq-2}.
		\end{proof}

\begin{thm}
	\label{thm-f}
Given $\epsilon>0$, there exists {$\bar{\theta} > 0$} such that $\forall \ \theta_0 >\bar{\theta}$,
$\|\hat{V}-V^*\|_\infty < \epsilon$, i.e., the policy $\hat{\pi}$ is $\epsilon$-optimal.
\end{thm}

\begin{proof}
The way $\hat{\pi}$ and $\pi^*$ are defined, we have $\hat{\pi}(i,\cdot) \rightarrow \pi^*(i,\cdot)$ as
$\theta_0\rightarrow\infty$. The claim now follows from Proposition~\ref{prop1} using the facts
that $\max_i \|\bar{g}(i,\cdot)\|_\infty \leq B <\infty$ for some scalar $B>0$, and that $\gamma \in [0,1)$.
	\end{proof}



\subsection{Convergence of the Slower Recursion}

The policy update in this algorithm is on the faster time scale and sees the value update as quasi-static. 
We have the following important result.
\begin{lem}
	\label{lemma-caac}
	Given $\epsilon>0$, there exists $\bar{\theta}$ sufficiently large such that for $\theta_0>\bar{\theta}$
	and $V \approx V_n$ (i.e., $V$ is tracking $V_n$ on a slower time scale), recursion \eqref{critic2} on the slower timescale satisfies
	$$\max_i\left|\min_u\sum_jp(i,u,j)(g(i,u,j) + \gamma V(j)) - V(i))\right| < \epsilon,$$
	for $n \geq 0$.
\end{lem}
\begin{proof}
	Follows as in Lemma 5.12 of \cite{KondaB}. 
\end{proof}

It follows that the iteration for $V_n$ on the slower time scale satisfies
\begin{eqnarray*}
&V_{n+1}(i) \approx V_n(i) + \bar{a}(n)\min_u[ \sum_j(p(i,u,j)g(i,u,j)\\
&+ \gamma V_n(j))- V_n(i) + N_{n+1}(i))],
\end{eqnarray*}
where, for any $i\in S, n \geq 0$, 
\begin{eqnarray*}
	&N_{n+1}(i) := g(i,\phi_n(i),\xi_n(i,\phi_n(i))) + \gamma V_n(\xi_n(i,\phi_n(i)))  \\
	&- E[g(i,\phi_n(i),\xi_n(i,\phi_n(i))) + \gamma V_n(\xi_n(i,\phi_n(i)))|\mathcal{F}_n],
\end{eqnarray*}
is a martingale difference sequence, {with the change from the previous definition of $N_n$ being that we use $V_n(\cdot)$ here in place of $V(\cdot)$}  
The approximate equality `$\approx$' accounts for the error
\[
\begin{split}
E[g(i,\phi_n(i),\xi_n(i,\phi_n(i))) + \gamma V_n(\xi_n(i,\phi_n(i)))|\mathcal{F}_n]\\
- \min_u[ \sum_j(p(i,u,j)g(i,u,j) + \gamma V_n(j))] ,
\end{split}
\]
which is seen from Lemma~\ref{lemma-caac} to be $O(\epsilon)$.

We thus have a result similar to Theorem 5.13, \cite{KondaB}:

\begin{thm}
	\label{main-ca}
	Given $\epsilon >0$, {$\exists \ \bar{\theta} \equiv \bar{\theta}(\epsilon) > 0$} such that for all $\theta_0 \geq \bar{\theta}$, $(V_n, \theta_n)$, $n\geq 0$, governed according to \eqref{critic2}-\eqref{actor2}, converges almost surely to the set
$$\{(V_{\pi_\theta},\theta) \mid V_\gamma (i) \leq V_{\pi_\theta}(i) \leq  V_\gamma (i) + \epsilon, \forall i \in S\},$$ 
where $V_\gamma(i), i\in S$ is the unique solution to the Bellman equation (\ref{be}).
\end{thm}

\begin{remark}
Note that Theorems 3 and 5 mention the existence of a $\bar{\theta}>0$. 
One way of getting such a $\bar{\theta}$ woulbe be as follows. Generically, the optimal $\pi(i,\cdot)$ would be at a corner of the probability simplex, corresponding to one control, say $a^*$ being chosen with probability $1$, i.e., $\theta(i,a*) = \infty$ and $\theta(i,a) = -\infty$ for $a \neq a^*$.Thus we can choose $theta_0$ so that for a prescribed small $\epsilon > 0$,
$$\left|\frac{e^{\theta_0}}{e^{\theta_0} + (|S|-1)e^{-\theta_0}} -1\right| < \epsilon.$$
\end{remark}

\begin{figure*}
	\centering
	\includegraphics[width=1\textwidth]{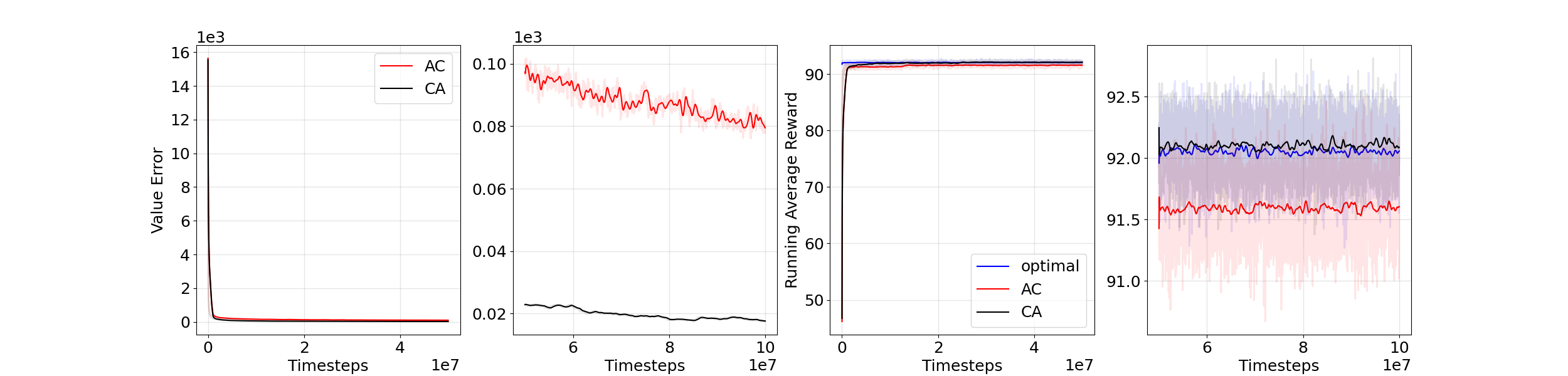}
	\caption{\(|S|=1000,|U|=6,\alpha_1=1,\beta_1=0.55,\alpha_2=1,\beta_2=0.55\)}
	\label{fig1}
\end{figure*}

\begin{figure*}
	\centering
	\includegraphics[width=1\textwidth]{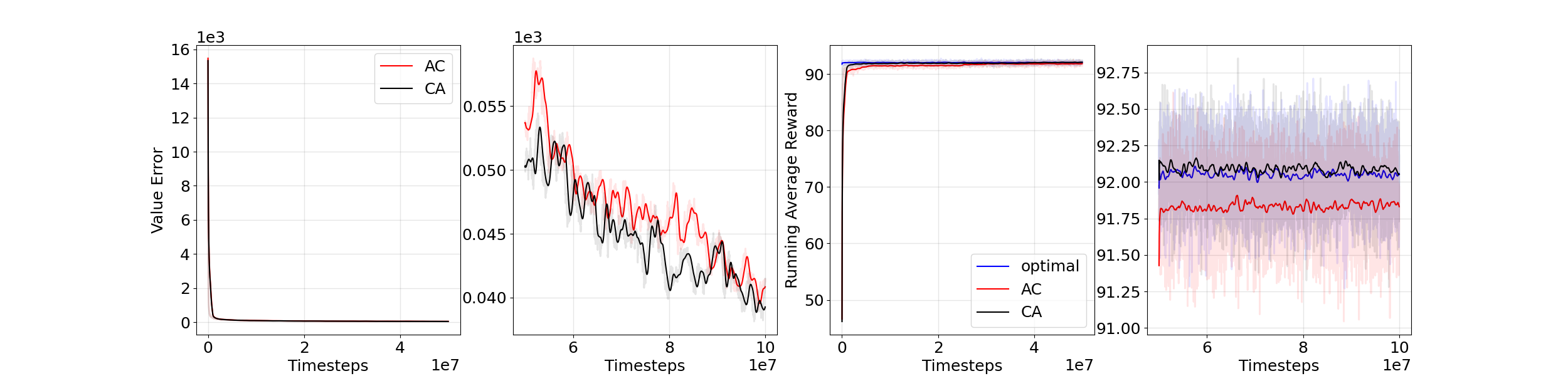}
	\caption{\(|S|=1000,|U|=6,\alpha_1=0.95,\beta_1=0.75,\alpha_2=0.75,\beta_2=0.55\)}
	\label{fig2}
\end{figure*}

\begin{figure*}
	\centering
	\includegraphics[width=1\textwidth]{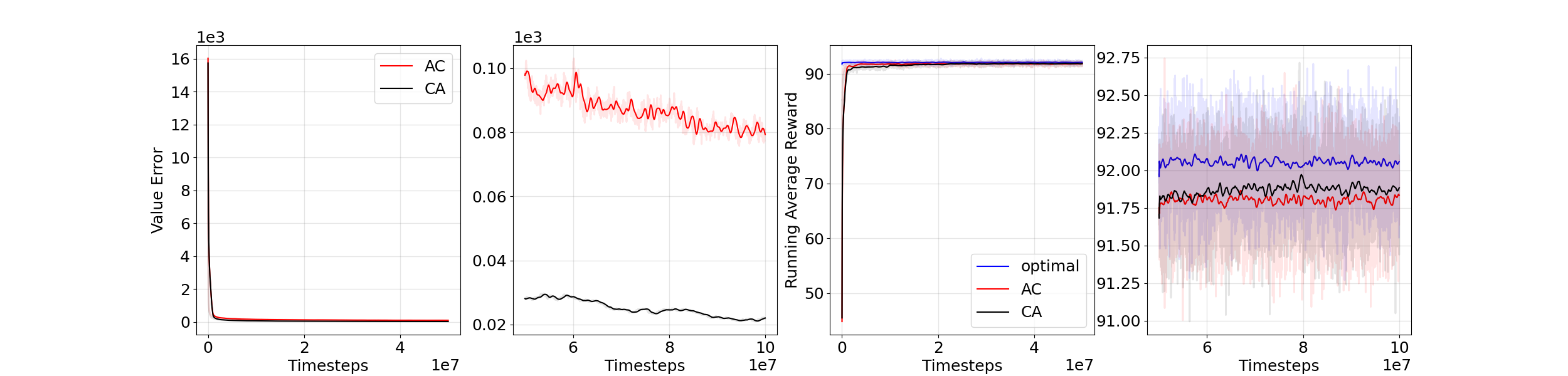}
	\caption{\(|S|=1000,|U|=6,\alpha_1=0.75,\beta_1=0.55,\alpha_2=0.95,\beta_2=0.75\)}
	\label{fig3}
\end{figure*}

\begin{figure*}
	\centering
	\includegraphics[width=1\textwidth]{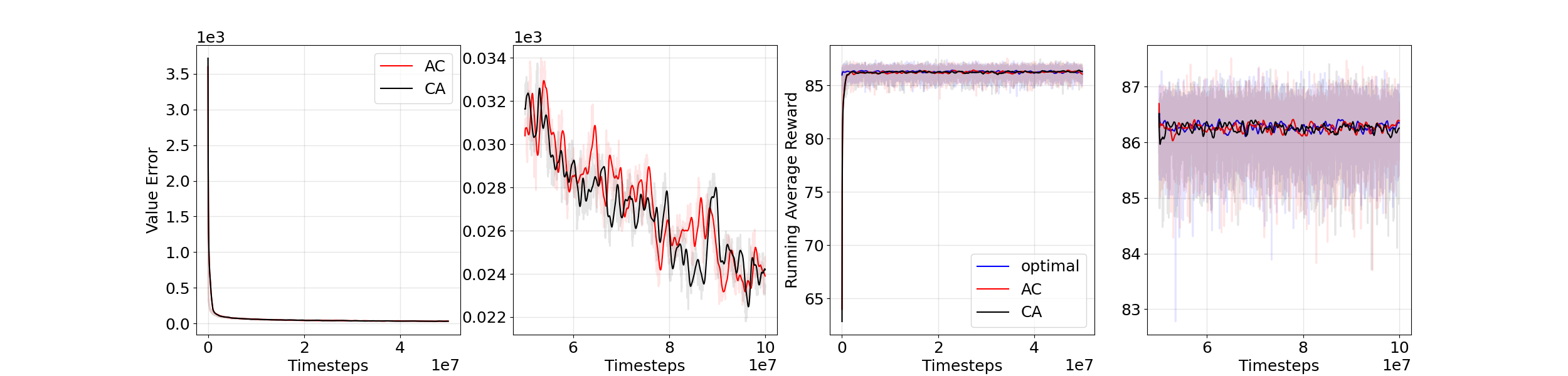}
	\caption{\(|S|=400,|U|=4,\alpha_1=0.95,\beta_1=0.75,\alpha_2=0.75,\beta_2=0.55\)}
	\label{fig4}
\end{figure*}

\begin{figure*}
	\centering
	\includegraphics[width=1\textwidth]{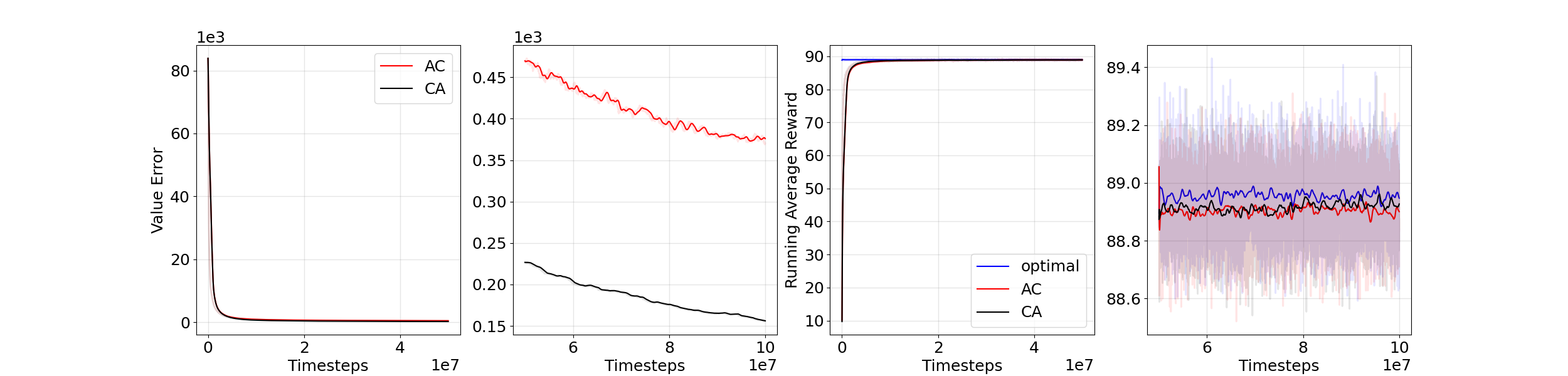}
	\caption{\(|S|=10000,|U|=8,\alpha_1=0.75,\beta_1=0.55,\alpha_2=0.95,\beta_2=0.75\)}
	\label{fig5}
\end{figure*}

\begin{figure*}
	\centering
	\includegraphics[width=1\textwidth]{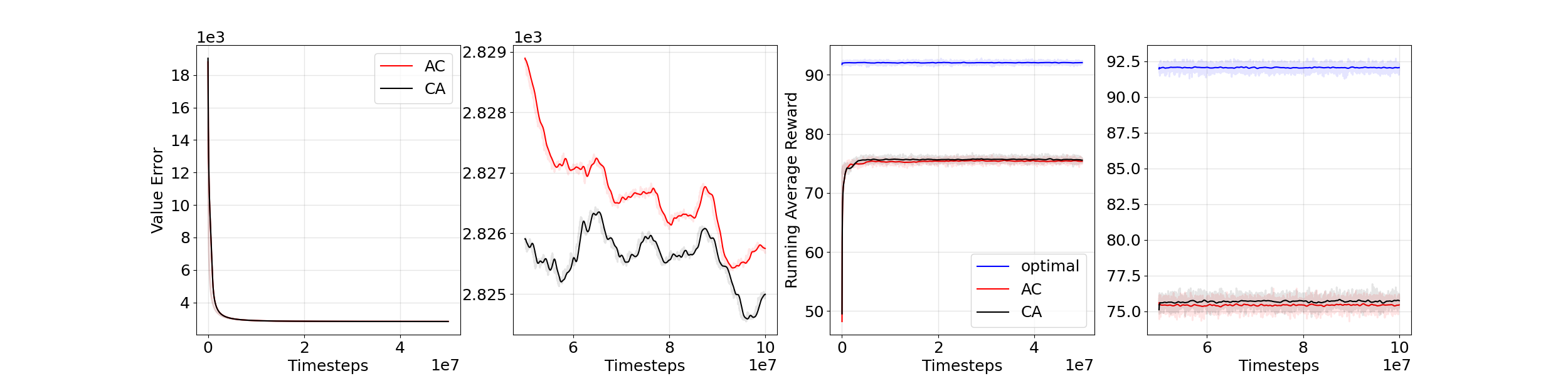}
	\caption{\(|S|=1000,|U|=6,\alpha_1=0.95,\beta_1=0.75,\alpha_2=0.75,\beta_2=0.55\)}
	\label{fig6}
\end{figure*}

\begin{figure*}
	\centering
	\includegraphics[width=1\textwidth]{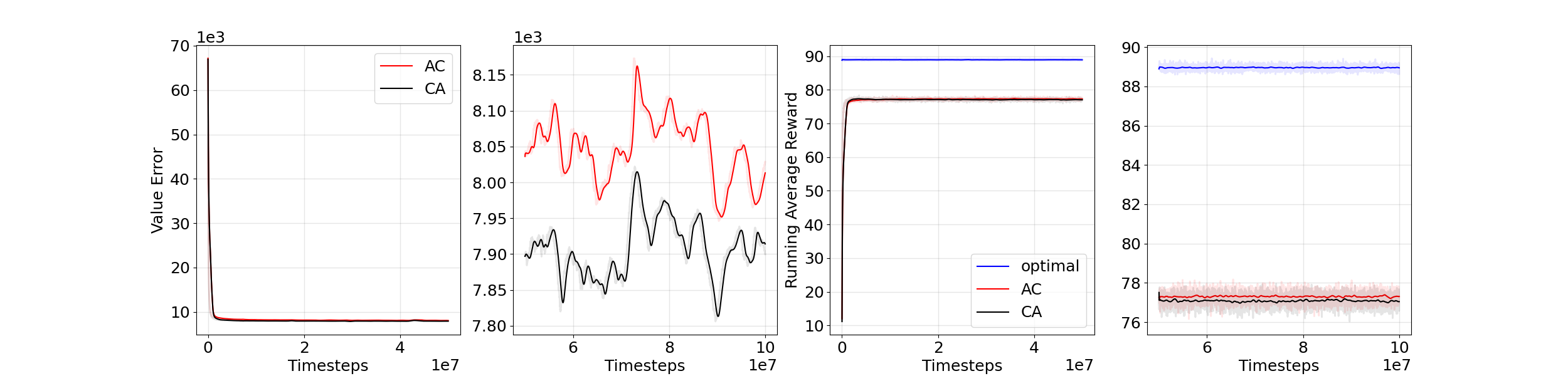}
	\caption{\(|S|=10000,|U|=8,\alpha_1=0.95,\beta_1=0.75,\alpha_2=0.75,\beta_2=0.55\)}
	\label{fig7}
\end{figure*}

{\small
	\begin{table*}
		\centering
		\begin{tabular}{|p{0.8cm}|p{1.5cm}|p{1.5cm}|p{1.5cm}|p{1.5cm}|p{2cm}|p{2cm}|}
			\hline
			\multirow{2}{*}{Exp} & \multicolumn{2}{c|}{Value Error} & %
			\multicolumn{2}{c|}{Average Reward} & \multicolumn{2}{c|}{Computational Time (hours)}\\
			\cline{2-7}
			& AC & CA & AC & CA & AC & CA\\
			\hline
			1 & $79.33$ $\pm4.23$ & $17.57$ $\pm1.59$ &$91.59$ $\pm1.13$ & $92.21$ $\pm0.41$ & $11.18$ $\pm0.36$ & $11.23$ $\pm0.26$\\
			\hline
			2 & $41.27$ $\pm4.63$ & $41.49$ $\pm4.00$ &$91.67$ $\pm0.90$ & $92.18$ $\pm0.49$ & $11.25$ $\pm0.22$ & $11.28$ $\pm0.32$\\
			\hline
			3 & $80.88$ $\pm5.15$ & $21.88$ $\pm0.85$ &$91.73$ $\pm0.70$ & $91.72$ $\pm0.49$ & $10.87$ $\pm0.20$ & $11.28$ $\pm0.30$\\
			\hline
			4 & $24.78$ $\pm1.64$ & $23.38$ $\pm2.22$ &$86.53$ $\pm0.34$ & $86.01$ $\pm1.28$ & $11.05$ $\pm0.20$ & $10.97$ $\pm0.06$\\
			\hline
			5 & $370.12$ $\pm5.08$ & $155.51$ $\pm4.04$ &$88.71$ $\pm0.18$ & $88.84$ $\pm0.25$ & $20.80$ $\pm0.30$ & $20.98$ $\pm0.81$\\
			\hline
			6 & $2825.67$ $\pm1.63$ & $2824.96$ $\pm0.72$ &$75.96$ $\pm0.42$ & $76.11$ $\pm0.71$ & $28.30$ $\pm0.43$ & $28.91$ $\pm0.93$\\
			\hline
			7 & $8028.80$ $\pm41.38$ & $7899.68$ $\pm51.49$ &$77.28$ $\pm0.43$ & $76.94$ $\pm0.46$ & $53.17$ $\pm1.12$ & $52.39$ $\pm0.50$\\
			\hline
		\end{tabular}
		\caption{Mean and standard deviation after \(10^8\) time-steps of average reward, value error and computational time}
		\label{table1}
	\end{table*}
}

\begin{figure*}
	\centering
	\includegraphics[width=1\textwidth]{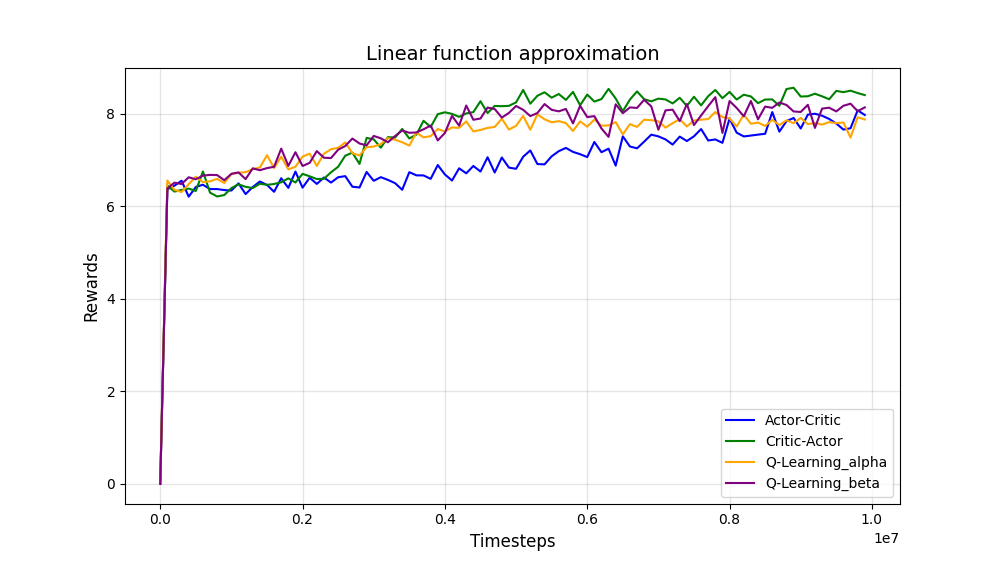}
	\caption{\(|S|=10000,|U|=9,\alpha=0.55,\beta=1\)}
	\label{fig8}
\end{figure*}
\begin{figure*}
	\centering
	\includegraphics[width=1\textwidth]{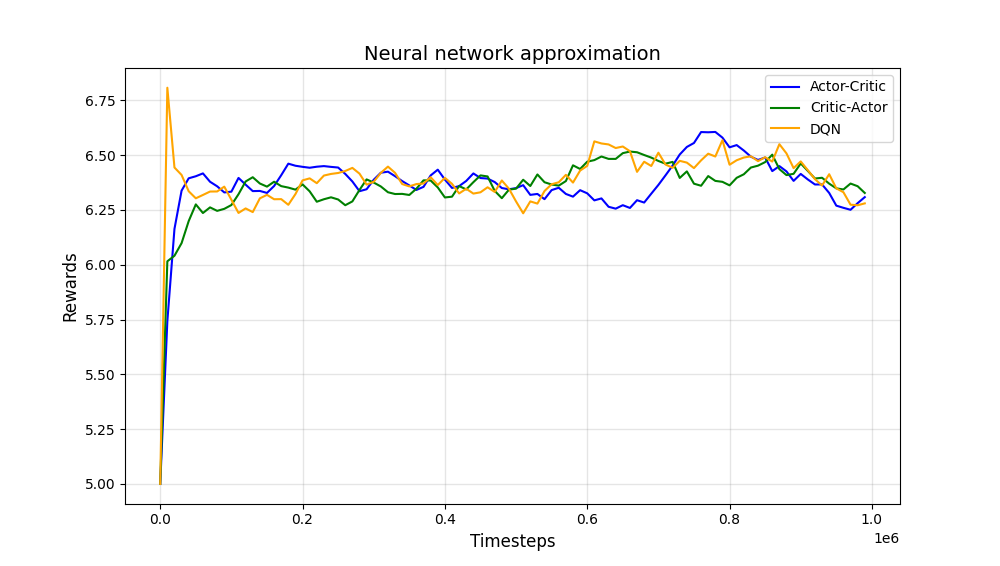}
	\caption{\(|S|=10000,|U|=9,\alpha=0.55,\beta=1\)}
	\label{fig9}
\end{figure*}

\begin{figure*}
	\centering
	\includegraphics[width=1\textwidth]{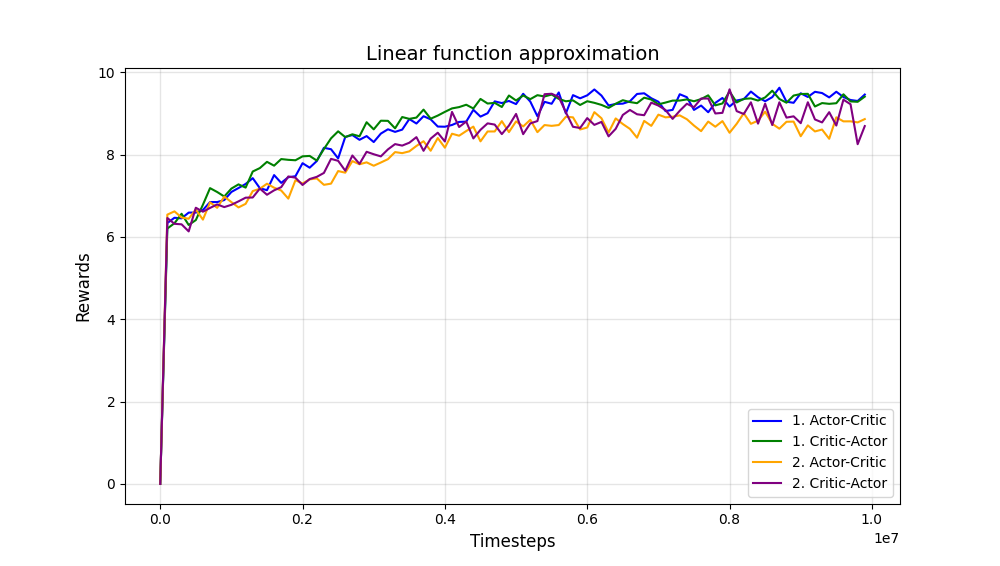}
	\caption{\(|S|=10000,|U|=9,\mbox{ Ideal Step-Sizes with } K_1=1,K_2=100000\)}
	\label{fig10}
\end{figure*}
\begin{figure*}
	\centering
	\includegraphics[width=1\textwidth]{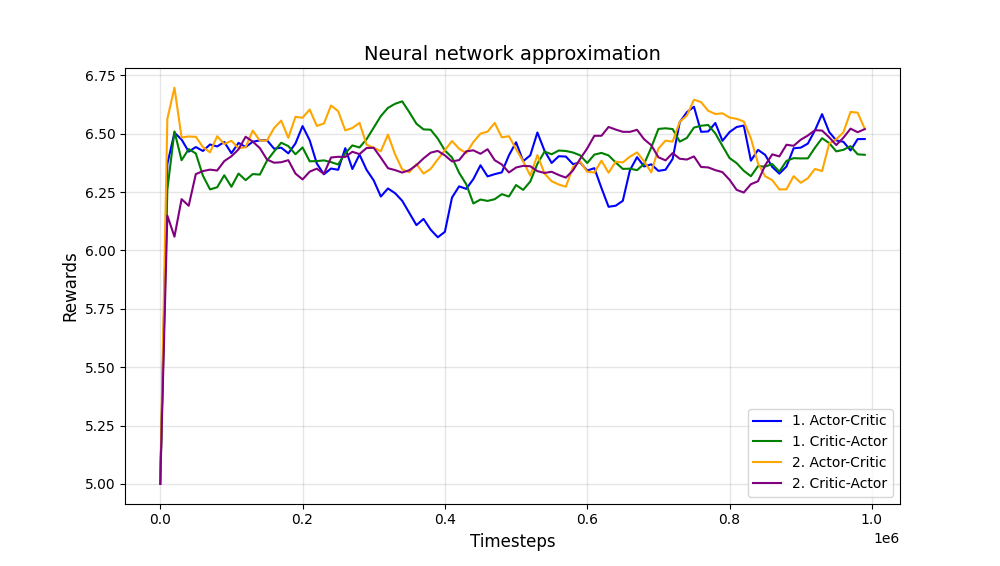}
	\caption{\(|S|=10000,|U|=9, \mbox{ Ideal Step-Sizes with } K_1=0.1,K_2=1000\)}
	\label{fig11}
\end{figure*}

\section{Numerical Results}
\label{numerical}
We show the results of experiments to study the empirical performance of the critic-actor (CA) algorithm \eqref{critic2}-\eqref{actor2} and it's comparison with the actor-critic (AC) algorithm (Algorithm 3 of \cite{KondaB}). Note that we compare the AC and CA algorithms in tabular form. This is both in order to put them on a common footing so that the comparison is legitimate, and also because in tabular form, both have rigorous convergence results. The difficulty in practice with the tabular form is its high dimensionality which can put it at a disadvantage, so the practical usage is invariably with function approximation. We also provide some comparisons with function approximations, but because of the possibility of multiple equilibria etc., there are other factors playing a role, so it is less conclusive evidence than the tabular form.

 We perform our experiments on Grid World settings of various sizes and dimensions. We present here the results of experiments for different step-size schedules. 
Since we use finite state-action MDP, there is always a unique optimal value function that we also compute for our experiments.
Our main performance metric is thus the value function error as obtained by the two algorithms. 
For this, we compute the optimal value function and use the Euclidean distance between the running value function estimates obtained from each algorithm and the optimal value function.

While both CA and AC algorithms aim to optimize the discounted reward objective, we also additionally study performance comparisons using the running average reward criterion obtained as a function of the number of iterates. The figures here show the performance comparisons in terms of (a) the value error estimates  and (b) the running average reward obtained from the two algorithms.
We observe in all cases that both algorithms converge rapidly. We also plot alongside in each figure the tail of the plots separately (after removing the initial transients) for a better depiction of the performance of the two algorithms when closer to convergence.
Finally, we also compare the amount of computational time required by the two algorithms on each of the settings. All results are shown by averaging over 5 different runs of each algorithm with different initial seeds. The standard error from these runs is also shown. 

In Figures 6 and 7, we study performance comparisons in the case when function approximation is used in the two algorithms. Specifically, the plots of CA and AC in Figure 7 are for the case when a neural network based function approximator is used while the corresponding plots in Figure 6 are for the case when linear function approximation is used in both the algorithms
Figures 1-5 show plots for experiments on the tabular setting (as described previously). 
Figures 1-3 are for experiments on a 3-dimensional grid world of size \(10 \times 10 \times 10\) (1000 states) and 6 actions (+x,-x,+y,-y,+z,-z). Figure 4 is for the case of a 2-dimensional grid of size \(20 \times 20\) and 4 actions (+x,-x,+y,-y). Figure 5 is for a 4-dimensional grid of size 10 in each dimension (10000 states) and 8 actions.\footnote{Note that both algorithms require a storage of $|S|^2\times |A|$. While actor is the faster recursion in CA, critic is slower here when compared to AC. One cannot infer, in general, computational superiority of one algorithm over another merely  on the basis of the number of states and actions used in the experimental setting.} In the experiments for the tabular setting, we use step-sizes of the form \(a(n)=\frac{1}{(\lfloor \frac{n}{100}\rfloor +1) ^{\alpha_1}},b(n)=\frac{1}{(\lfloor \frac{n}{100}\rfloor +1) ^{\beta_1}}\) for AC and \(a(n)=\frac{1}{(\lfloor \frac{n}{100}\rfloor +1) ^{\alpha_2}},b(n)=\frac{1}{(\lfloor \frac{n}{100}\rfloor +1) ^{\beta_2}}\) for CA. Note that we decrease the step-sizes every 100 time-steps for faster convergence. We mention here that the step-sizes we use for our experiments follow Assumptions 1(i)-(iii) but not necessarily the ideal step-size requirements in Assumptions 1(iv)-(v) and 2(ii) under which convergence of our algorithm was theoretically shown. 
However, for consistency with the theoretical development, we also show the results of experiments in Figures 10 and 11 where we use step-sizes that satisfy all the conditions in Assumptions 1 and 2 of the paper. We observe good performance of CA as with AC when these step-sizes are used as well.
In all the Figures 1-5, we can see that the value function error successfully goes to almost zero. 

When we approximate the value function using linear function approximation, we decrease the step-sizes every 100 time-steps (same as in the tabular case). We take features for state \(i \in S\) where the value at the \(\lfloor\frac{i}{10}\rfloor\)-th position is 1 and the rest are 0. Note\label{key} that we use features such that the basis functions are linearly independent (see \cite{TsitsiklisV}). When we approximate the value function using a neural network, we decrease the step-sizes every time-step. This is needed for obtaining good performance since otherwise with larger step-sizes that would result in a slow decrease, we observe the algorithm does not show good performance. We use a fully connected feedforward neural network with two hidden layers, 10 neurons per layer, $\tanh$ activation function, and the single number state as input. We run all our experiments for \(10^8\) time-steps. The mean and standard deviation over 5 runs after \(10^8\) time-steps are shown in Table \ref{table1} in the form \(\mu \pm \sigma^2\), where \(\mu\) and \(\sigma^2\) are the mean and standard deviation respectively. 

In Figures 8 and 9, we study performance comparisons when function approximation is used in both actor and critic and we compare these algorithms with Q-learning with function approximation. We use policy gradient for the actor update and TD(0) for the critic update. The experiments are run on a two-dimensional grid of size $100 \times 100$, and 9 actions. Along each dimension, the agent has the option to go forward, backward or stay in the same position. The actions are $((+1,+1),(+1,0),(+1,-1),(0,+1)$, $(0,0),(0,-1),(-1,+1),(-1,0),(-1,-1))$. Figure 8 is for the case where linear function approximation is used in the critic and Boltzman policy (which is a softmax over linear functions) is used for the actor. Note now that the parameter $\theta$ is no longer a function of the tuple $(i,a)$. In the critic we take features for state \(i \in S\) where the value at the \(\lfloor\frac{i}{100}\rfloor\)-th position is 1 and the rest are 0. The features for the Boltzman policy for state-action pair $(i,a)$ is the feature of state $i$ used in the critic at $a$-th position. We use step-sizes of the form \(a(n)=\frac{0.001}{(\lfloor \frac{n}{100000}\rfloor +1) ^{\beta}},b(n)=\frac{0.001}{(\lfloor \frac{n}{100000}\rfloor +1) ^{\alpha}}\) for both AC and CA. 
We show two plots for Q-learning in Figure 8, namely Q-Learning$\_$alpha and Q-Learning$\_$beta, that use step sizes $a(n)$ and $b(n)$ of AC, respectively, in the two cases. From Figure 8, we see similar performance for CA and Q-Learning with linear function approximation and both algorithms show superior performance than the AC algorithm. 

Figure 9 is for the case when neural network approximators are used for both actor and critic. The neural network architecture for critic is the same as for Experiment 7 (see above), except that 2 numbers (row and column number of the grid) is given as input instead of a single state number. The architecture for the actor is the same as for critic with the last layer as softmax over actions. We use step-sizes of the form \(a(n)=\frac{0.01}{(\lfloor \frac{n}{1000}\rfloor +1) ^{\beta}},b(n)=\frac{0.01}{(\lfloor \frac{n}{1000}\rfloor +1) ^{\alpha}}\) for both AC and CA. We compare here the performance of the algorithms with the Deep Q-Network \cite{Mnih} (that is Q-learning with neural network function approximators). 
From Figure 9, we see similar performance for CA, AC and DQN. However, DQN takes significantly more computational time than CA and AC. We also see greater variance in the  performance of AC. This further indicates that CA is a promising algorithm that can be considered over other algorithms.

Finally, in Figures 10 and 11, we compare the performance of AC and CA with function approximation using the ideal step-sizes, i.e., those that satisfy all of the conditions namely Assumptions 1 (i)-(v) and Assumptions 2 (i)-(ii). Each figure here has two sets of two plots each (one for CA and the other for AC) that are obtained using two different step-size schedules that both satisfy Assumptions 1 and 2.
For simplicity, we refer to these  plots by
1. Actor-Critic, 1. Critic-Actor, 2. Actor-Critic and 2. Critic-Actor, respectively. The step-sizes for 1. Actor-Critic and 1. Critic-Actor have the form \(a(n)=\frac{K_1}{\lfloor \frac{n}{K_2}\rfloor +1},b(n)=\frac{K_1(\log{\lfloor \frac{n}{K_2}\rfloor +2)}}{\lfloor \frac{n}{K_2}\rfloor +2}\) for both AC and CA. 
Further, the step-sizes used for 2. Actor-Critic and 2 .Critic-Actor are of the form \(a(n)=\frac{K_1}{\log{(\lfloor \frac{n}{K_2}\rfloor +2)}(\lfloor \frac{n}{K_2}\rfloor +2)},b(n)=\frac{K_1}{\lfloor \frac{n}{K_2}\rfloor +1}\) for both AC and CA. Figure 10 concerns the case when linear function approximators are used, where all the features are the same as described for Experiment 8 (see above). On the other hand, Figure 11 uses neural network approximators where the neural network architecture is the same as for Experiment 9 (see above).
 
From the figures and the table, it can be seen that the performance of our Critic-Actor algorithm is as good (in fact, slightly better overall) as compared to the well-studied Actor-Critic algorithm. While there is no clear winner between the two schemes (CA vs. AC), our paper presents for the first time an important and previously unstudied class of algorithms -- the critic-actor algorithms that we believe will be studied significantly more in the future.

\section{Conclusions}
\label{conclusion}

{We presented for the first time an important and previously unstudied class of algorithms -- the critic-actor algorithms, that holds much promise.}
Like actor-critic, these are also two-timescale algorithms, however, where the value critic is run on a slower timescale as compared to the policy actor.
Whereas actor-critic algorithms emulate policy iteration, we argue that critic-actor mimics value iteration. We proved the convergence
of the algorithm. Further, we showed the results of several experiments on a range of Grid World settings and observed that our critic-actor algorithm shows 
similar or slightly better performance as compared to the corresponding actor-critic algorithm. 
We hope that our work will result in further research in this hitherto unexplored  direction. For instance, it would be
of interest to study further the critic-natural actor algorithms such as \cite{ShalabhSGL} with time scales reversed. A further possibility would be to explore constrained critic-actor algorithms with the timescales of the actor and the critic reversed from the constrained actor-critic algorithms devised such as in  \cite{Borkar-Constrained, Shalabh, ShalabhL2, Shalabh-Book}.

\end{document}